\documentclass[runningheads]{llncs}

\usepackage{url}
\usepackage{hyperref}
\hypersetup{urlcolor=blue}
\usepackage{graphicx}
\usepackage{graphics}
\usepackage{times}
\usepackage{mathptmx}
\usepackage{mathtools}
\usepackage{amsmath}
\usepackage{amssymb}
\usepackage{amsfonts}
\usepackage[utf8]{inputenc}
\usepackage[small]{caption}
\usepackage{booktabs}
\usepackage{multirow}
\usepackage{colortbl}
\usepackage{tikz}
\usepackage{pgfplots}
\usepackage{soul}
\usepackage{xfakebold}
\usepackage[misc,geometry]{ifsym}

\newcommand{\setBoldness}[1]{\def\fake@bold{#1}}
\newcommand{\savefootnote}[2]{\footnote{\label{#1}#2}}
\newcommand{\repeatfootnote}[1]{\textsuperscript{\ref{#1}}}
\newcommand{\fbseries}{\unskip\setBold\aftergroup\unsetBold\aftergroup\ignorespaces}
\pgfplotsset{width=7.5cm,compat=1.12}
\usepgfplotslibrary{fillbetween}

\newcommand\norm[1]{\lVert#1\rVert}

\DeclareMathOperator*{\argmax}{arg\,max}

\usepackage{algorithm}
\usepackage{algpseudocode}
\algnewcommand{\Initialize}[1]{
  \State \textbf{Initialize:}
  \Statex \hspace*{\algorithmicindent}\parbox[t]{.8\linewidth}{\raggedright #1}
}
\algnewcommand{\Input}[1]{
  \State \textbf{Input:} {\raggedright #1}
}
\algnewcommand{\Output}[1]{
  \State \textbf{Output:} {\raggedright #1}
}

\begin{document}

\title{Are Transformers More Robust? Towards Exact Robustness Verification for Transformers\thanks{\includegraphics[height=0.25cm]{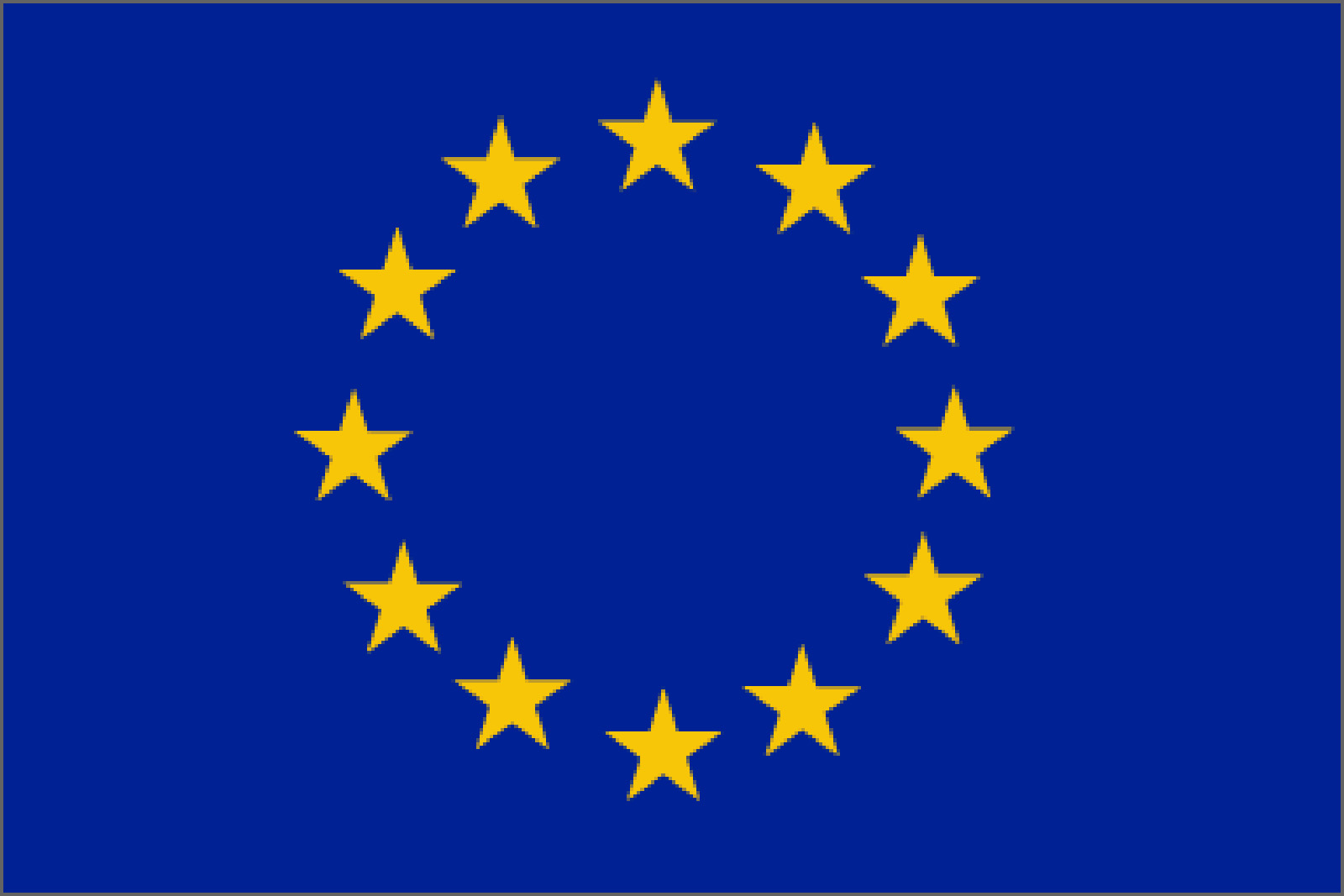} This project has received funding from the European Union’s Horizon 2020 research and innovation programme under grant agreement No 956123 - FOCETA.}}

\titlerunning{Towards Exact Robustness Verification for Transformers}

\author{Brian Hsuan-Cheng Liao\inst{1,2} \Letter
\and
Chih-Hong Cheng\inst{3}
\and
Hasan Esen\inst{1} 
\and 
Alois Knoll\inst{2}
}

\authorrunning{B. H.-C. Liao et al.}

\institute{
DENSO AUTOMOTIVE Deutschland GmbH, 85386 Eching, Germany \\\email{\{h.liao, h.esen\}@eu.denso.com}
\and
Technical University of Munich, 85748 Garching, Germany 
\\\email{knoll@in.tum.de}
\and
Fraunhofer IKS, 80686 Munich, Germany
\\\email{chih-hong.cheng@iks.fraunhofer.de}
\vspace{-3mm}
}

\maketitle

\begin{abstract}
As an emerging type of Neural Networks (NNs), Transformers are used in many domains ranging from Natural Language Processing to Autonomous Driving. In this paper, we study the robustness problem of Transformers, a key characteristic as low robustness may cause safety concerns. Specifically, we focus on Sparsemax-based Transformers and reduce the finding of their maximum robustness to a Mixed Integer Quadratically Constrained Programming (MIQCP) problem. We also design two pre-processing heuristics that can be embedded in the MIQCP encoding and substantially accelerate its solving. We then conduct experiments using the application of Land Departure Warning to compare the robustness of Sparsemax-based Transformers against that of the more conventional Multi-Layer-Perceptron (MLP) NNs. To our surprise, Transformers are not necessarily more robust, leading to profound considerations in selecting appropriate NN architectures for safety-critical domain applications.
\vspace{-2mm}
\end{abstract}

\section{Introduction}
\label{sec:introduction}
\vspace{-2mm}
Over the past decade, Neural Networks (NNs) have been widely adopted for many applications, including automated vehicles (AVs)~\cite{bojarski2016end}. Lately, as an emerging type of NNs, Transformers~\cite{vaswani2017attention} are often found to be the most effective models, compared to the more conventional Multi-Layer Perceptrons (MLPs) or their convolutional and recurrent variants~\cite{dosovitskiy2021image}, thereby gradually replacing them in these applications. For instance, Tesla and Cruise use Transformers in their perception units~\cite{cruise2021underthehood,tesla2022aiday}. However, most of the studies and discussions focus on evaluating NNs' accuracy. Parallel research has shown that NNs often lack robustness against input changes such as adversarial attacks or domain shifts, hence hindering the overall dependability of the NN-based applications~\cite{goodfellow2015explaining,hu2022human,su2019one}. 

The above background naturally prompts a question of whether Transformers are more robust than MLPs, given the often better accuracy and wide applications. To answer this question, we study the maximum robustness of the NNs against local input perturbations commonly modeled by $l_p$-distances, i.e., exact robustness verification (where $p$ can be $1,2, ..., \infty$). Ultimately, the goal is to hold a direct comparison of the robustness of the two kinds of NNs, Transformers and MLPs, and gain insights from the results. 

In the literature, research efforts have been made to enable exact robustness verification for MLPs, particularly ones with feed-forward layers and ReLU activation function~\cite{cheng2017maximum,ehlers2017formal,lomuscio2017approach,tjeng2019evaluating}. The main approach is to employ an optimization framework, encode the NNs' architecture into the constraints, and calculate the exact robustness within some admissible input perturbation region~\cite{cheng2017maximum,tjeng2019evaluating}. However, there is still a gap for Transformers' exact robustness verification due to the more complex operations in this kind of NNs, namely the dot product between variables and the activation function in the Multi-Head Self-Attention (MSA) block~\cite{vaswani2017attention}. The existing (small volume of) works handle these operations with approximations during verification yet at the expense of verification precision~\cite{bonaert2021fast,shi2020robustness} (more details can be found in Section~\ref{sec:related_work}).

Our work attempts to close the gap towards exact robustness verification for Transformers but provides merely an interim solution. To elaborate, having formulated a Mixed Integer Programming (MIP)-based optimization problem, we focus on the Transformers that use Sparsemax (instead of Softmax) for MSA activation. This allows for precisely encoding the NN into the MIP, or more particularly, Mixed Integer Quadratically Constrained Programming (MIQCP) due to the remaining quadratic terms. We provide a comparison study to show that the Sparsemax-based Transformers perform similarly to their Softmax-based counterparts. Then, to faster solve the MIQCP, we devise two pre-processing heuristics that lead to a total speedup of an order of magnitude. Notably, these heuristics are not restricted to our work but can be applied to related studies. 

We perform the experiments using a Lane Departure Warning (LDW) application, which is widely adopted in AVs. Such an LDW application can be used for human driving assistance or run-time monitoring on separate automated driving functions. Essentially, an LDW application is a time-series classification and regression task. The embedded model, usually an NN, has to predict the direction of and the time to a potential lane departure, given a sequence of past driving information such as the ego vehicle velocity and estimated time to collision against adjacent vehicles. Our methodology and experimental results, though limited (similar to the exact robustness verification works focusing on ReLU-based MLPs), demonstrate that Sparsemax-based Transformers tend to be less robust than similar-sized MLPs despite generally higher accuracy. Resonating the government publications~\cite{ec2021aiact} and industrial guidelines~\cite{poretschkin2023ai}, our findings suggest that conducting thorough studies and providing rigorous guarantees on metrics beyond accuracy is crucial before deploying an NN-based application. In summary, our contributions include the following:
\begin{itemize}
    \item To implement exact robustness verification for Sparsemax-based Transformers;
    \item To propose two accelerating heuristics for related robustness verification studies;
    \item To benchmark ATN and MLP accuracy and robustness with an industrial application (i.e., LDW).
\end{itemize}

The rest of the paper is organized in the following way. Section~\ref{sec:related_work} browses the relevant literature emphasizing verification methods for general NNs and Transformers; Section~\ref{sec:preliminaries} introduces the branch of Transformers concerned in this paper. Section~\ref{sec:methodology} details the problem formulation and our heuristics for robustness verification, whose effectiveness and efficiency are demonstrated and discussed in Section~\ref{sec:experiments}. Lastly, Section~\ref{sec:conclusion} concludes with a few final remarks.
\section{Related Works}
\label{sec:related_work}
This section overviews related works, focusing on robustness verification for ReLU-based MLPs (i.e., piece-wise linear feed-forward NNs) and Transformers.

\subsection{Robustness Verification for Neural Networks}
\label{subsec:robustness_verification_nns}
Following the common categorization~\cite{tjeng2019evaluating}, we introduce two main branches of verification methods: complete and incomplete. To illustrate the difference, we assume an \textit{adversarial polytope} to be the exact set of NN outputs resulting from the norm-bounded perturbation region. To assert the robustness of the NN within the perturbation region, complete methods handle the adversarial polytope directly, attaining an adversarial example or a robustness certificate for each query when given sufficient processing time. These methods usually apply Mixed Integer Programming (MIP)~\cite{cheng2017maximum,lomuscio2017approach,tjeng2019evaluating} or Satisfiability Modulo Theory (SMT)~\cite{ehlers2017formal,huang2017safety,katz2017reluplex}, which in turn utilizes Linear Programming (LP) or Satisfiability (SAT) solvers with accelerating techniques such as interval analysis~\cite{cheng2017maximum,ehlers2017formal,tjeng2019evaluating} or region partitioning~\cite{everett2021robustness} in a Branch-and-Bound (BnB) fashion~\cite{wang2021betacrown}. By contrast, incomplete methods reason upon an outer approximation of the adversarial polytope. Such reasoning typically results in faster verification time, yet possibly some robust queries being evaluated non-robust due to the over-approximation. Common methods in this branch include duality~\cite{wong2018provable}, abstract interpretation~\cite{gehr2018ai2} and Semi-Definite Programming (SDP)~\cite{wang2021betacrown}. For more details, interested readers are referred to the survey paper~\cite{huang2020survey}.

\subsection{Robustness Verification for Transformers}
\label{subsec:robustness_verification_transformers}
Transformers are typically more challenging to verify because they contain more complex operations than other NNs, such as MLPs. We are only aware of two lines of existing works~\cite{bonaert2021fast,shi2020robustness}, both conducted with sentiment classification in Natural Language Processing. In~\cite{shi2020robustness}, the authors calculate linear intervals for all operations in a Transformer to find the lower bound of the difference between the Softmax values of the ground-truth class and the most-probable-other-than-ground-truth class. If the computed lower bound is larger than zero (within an admissible perturbation region), the model is guaranteed robust since the prediction remains unchanged. As mentioned, the effectiveness of such approaches depends heavily on how well the lower bound approximates the actual difference of the Softmax values. Holding a similar strategy, the most recent work~\cite{bonaert2021fast} applies abstract interpretation and suggests techniques, such as noise symbol reduction and Softmax summation constraint, to achieve better speed and precision during verification.

Our work differs from these verification frameworks in two ways. First, we find Transfomers' maximum robustness (i.e., minimum adversarial perturbation) through precise MIP encoding and fasten the procedure with novel tactics, allowing us to compare the robustness with common MLPs. Second, we conduct experiments with LDW, an industrial application not used before in robustness verification studies.
\section{Preliminaries}
\label{sec:preliminaries}
This section provides a brief description of the Transformer under verification. For a more elaborate illustration of general Transformers, readers are referred to~\cite{dosovitskiy2021image,vaswani2017attention}.

\begin{figure}[t!]
	\centering
	\includegraphics[width=0.6\linewidth]{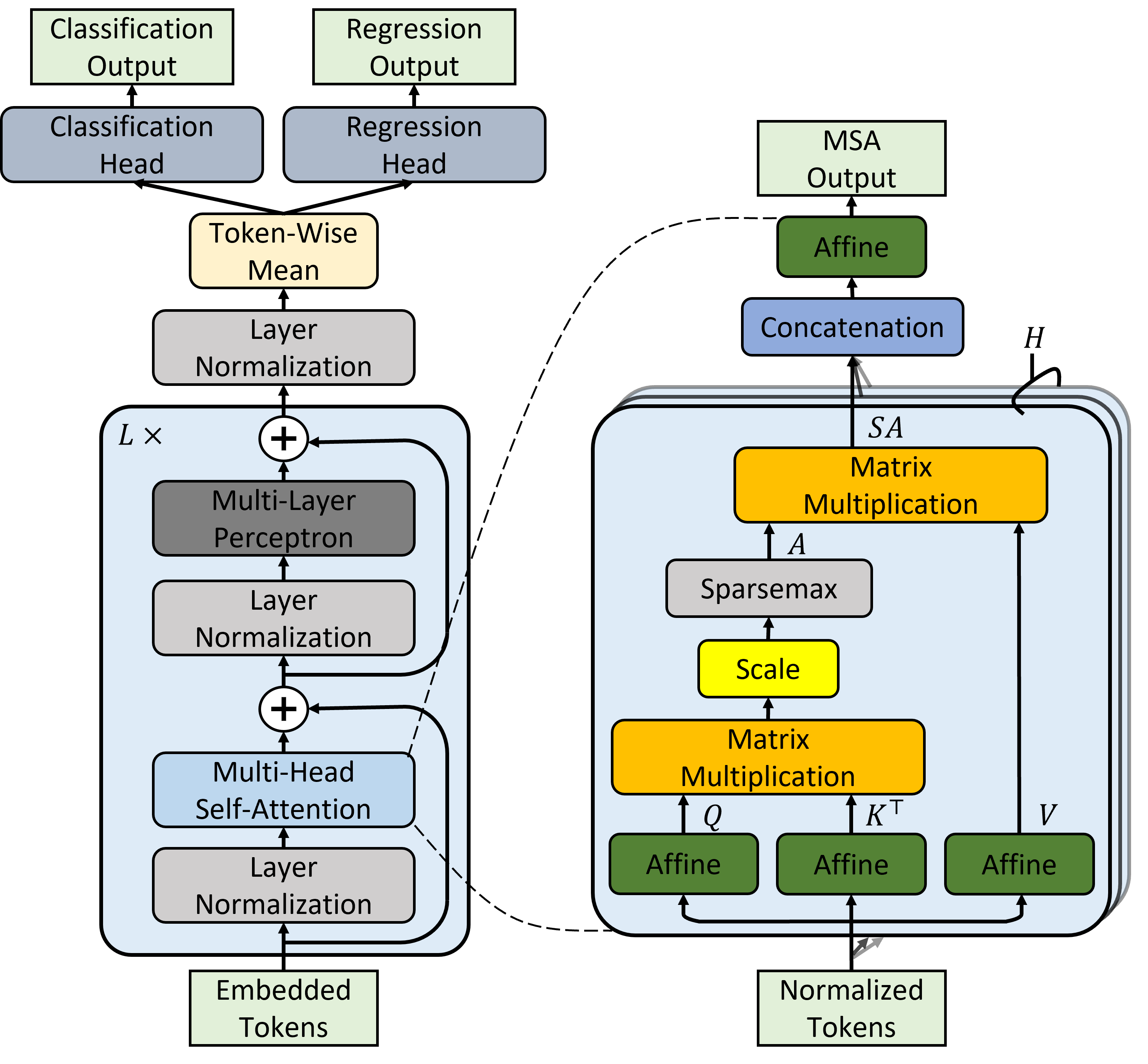}
	\caption{The Sparsemax-based Transformer under exact robustness verification~\cite{dosovitskiy2021image,vaswani2017attention}.}
	\vspace{-3mm}
	\label{fig:atn}
\end{figure}

As shown in Fig.~\ref{fig:atn}, a Transformer typically processes an array of embedded tokens with alternating blocks of MSA and MLP, which are both preceded by Layer Normalization (LN) and followed by residual connections\footnote{Placing LN before MSA and MLP is found to give better network performance than placing it after residual addition~\cite{xiong2020layer}.}. Then, after another LN and token-wise mean extraction, the network is appended with suitable affine heads for downstream predictions such as classification (CLS) and regression (REG). Mathematically, given an input $\textbf{x} \in \mathbb{R}^{N \times D}$, in which $N$ is the number of tokens and $D$ the dimension of features, we write:
\begin{align}
    \textbf{z}^0 &= \textbf{x} 
    \label{eq:z_0} \\
    \widehat{\textbf{z}^{\ell}} &= \textsf{MSA}(\textsf{LN}(\textbf{z}^{\ell-1})) + \textbf{z}^{\ell-1},
    \label{eq:msa} \\
    \textbf{z}^\ell &= \textsf{MLP}(\textsf{LN}(\widehat{\textbf{z}^{\ell}})) + \widehat{\textbf{z}^{\ell}},  
    \label{eq:mlp} \\
    f^\textsf{CLS}(\textbf{x}) & = \textsf{CLS}(\widetilde{\textbf{z}^{L}}) \in \mathbb{R}^C, 
    \label{eq:f_cls} \\
    f^\textsf{REG}(\textbf{x}) &= \textsf{REG}(\widetilde{\textbf{z}^{L}}) \in \mathbb{R}
    \label{eq:f_reg},
\end{align}
in which $\widetilde{\textbf{z}^{L}} \in \mathbb{R}^D$ is the token-wise mean of $\textsf{LN}(\textbf{z}^{L}) \in \mathbb{R}^{N \times D}$ , $C$ is the number of predefined classes and $\ell=1, \dots, L$ is the layer index. We define the functions in the following.

\begin{figure}[b!]
    \centering
    \includegraphics[width=0.35\linewidth]{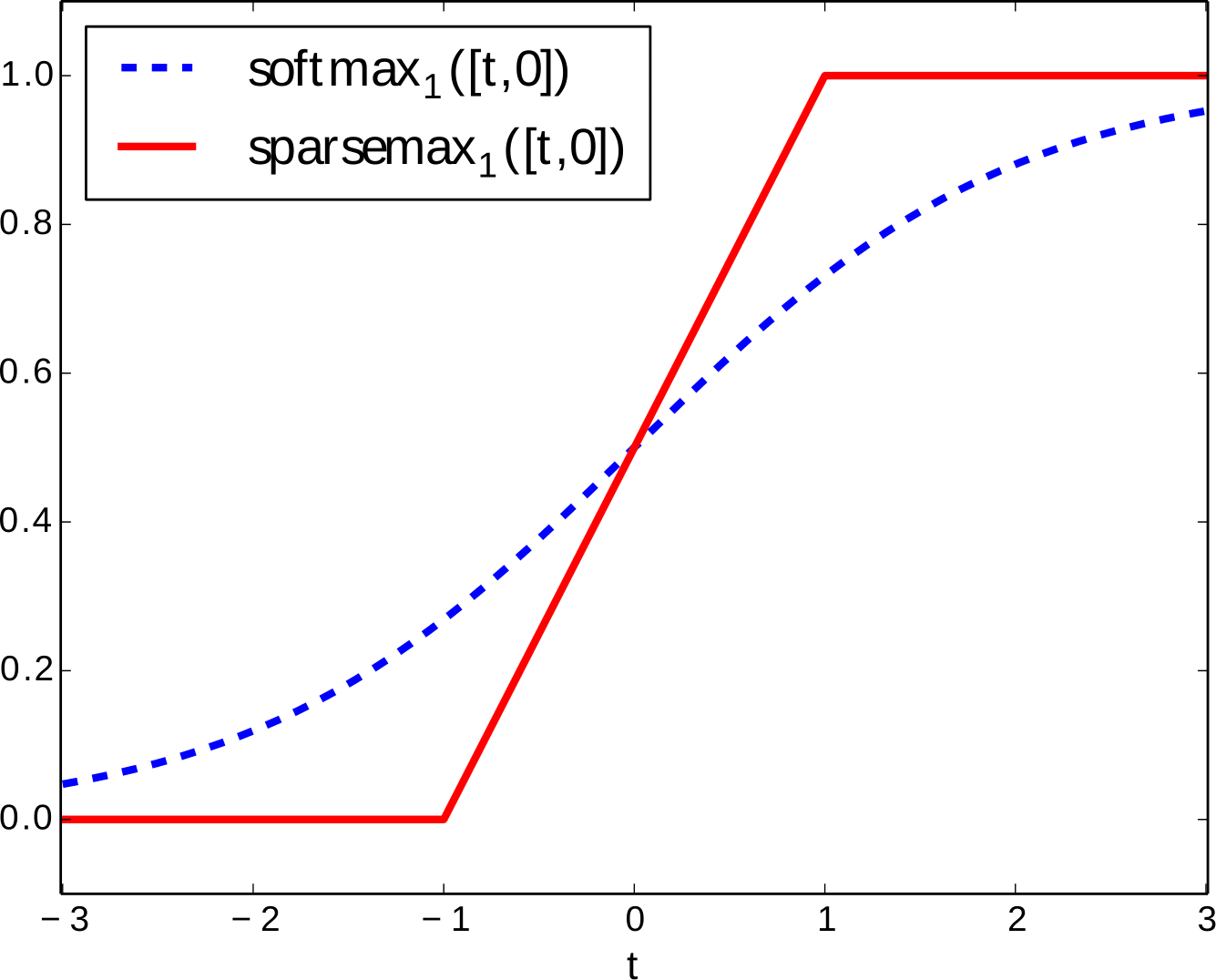}
    \caption{Softmax vs. Sparsemax, given an input vector $\textbf{u}=[t, 0] \in \mathbb{R}^2$ (adapted from~\cite{martins2016softmax}).}
    \label{fig:sparsemax}
\end{figure}

\begin{algorithm}[b!]
\caption{Calculate Sparsemax activation~\cite{martins2016softmax}}
\label{alg:sparsemax}
\begin{algorithmic}[1]
   \Input{\textbf{u} $\in \mathbb{R}^D$}
   \State Sort \textbf{u} into $\hat{\textbf{u}}$, where $\hat{\textbf{u}}_1 \ge \dots \ge \hat{\textbf{u}}_D$ \label{alg:sparsemax:sort}
   \State Find $s(\hat{\textbf{u}}) := \max \left\{ k \in [1, D] \; | \; 1+k\hat{\textbf{u}}_k > \sum_{j=1}^{k} \hat{\textbf{u}}_j \right\}$ \label{alg:sparsemax:find}
   \State Define $\tau(\hat{\textbf{u}}) = \frac{\left(\sum_{j=1}^{s(\hat{\textbf{u}})} \hat{\textbf{u}}_j \right) - 1}{s(\hat{\textbf{u}})}$ \label{alg:sparsemax:tau}
   \Output{$\textbf{p} \in \mathbb{R}^D$, where $\textbf{p}_i = \max (\textbf{u}_i - \tau(\hat{\textbf{u}}),\, 0) $}
\end{algorithmic}
\end{algorithm}

As introduced, we consider Sparsemax-based MSA with the definition:
\begin{align}
    [\textbf{Q}_h, \textbf{K}_h, \textbf{V}_h] &= \textbf{z} \textbf{W}^\textsf{QKV}_h, 
    \label{eq:qkv} \\
    \textbf{A}_h &= \textsf{\textsf{Sparsemax}}\left(\textbf{Q}_h\textbf{K}^\top_h / \sqrt{D_H}\right), 
    \label{eq:attn_matrix} \\
    \textsf{SA}_h(\textbf{z}) &= \textbf{A}_h\textbf{V}_h , 
    \label{eq:sa} \\
    \textsf{MSA}(\textbf{z}) &= [\textsf{SA}_1(\textbf{z}), \dots , \textsf{SA}_H(\textbf{z})] \, \textbf{W}^{\textsf{MSA}}, 
    \label{eq:msa_detailed}
\end{align}
where $\textbf{z} \in \mathbb{R}^{N \times D}$ is a general input matrix, $\textbf{Q}_h, \textbf{K}_h, \textbf{V}_h \in \mathbb{R}^{D_H}$ are the query, key and value matrices for the $h$-th self-attention head, $H$ is the number of self-attention heads, $D_H$ is the dimension of each head, $\textbf{W}^\textsf{QKV}_h \in \mathbb{R}^{D \times (D_H \cdot 3)}$ and $\textbf{W}^{\textsf{MSA}} \in \mathbb{R}^{(D_H \cdot H) \times D}$ are the trainable weights, and $h=1, \dots, H$ is the head index. Essentially, Sparsemax projects an input vector $\textbf{u} \in \mathbb{R}^D$ to an output vector $\textbf{p} \in \mathbb{R}^D$, where $\textbf{p}_1 + \dots + \textbf{p}_D=1$ and $\textbf{p}_i \ge 0$ for $i=1, \dots, D$ (i.e., a probability simplex)\savefootnote{fn:vector}{We write here in vector form (c.f.~\eqref{eq:z_0}-\eqref{eq:mlp}) as the operations are applied vector-wise essentially.}. Technically, it can serve as a piece-wise linear approximation to Softmax~\cite{martins2016softmax}. Algorithm~\ref{alg:sparsemax} and Fig.~\ref{fig:sparsemax} provide the calculation steps for a closed-form solution and the visualization of a 2D input-output relation. 

To proceed with the Transformer architecture, MLP is a two-layer NN with ReLU activation, written as $\textsf{MLP}(\textbf{v}) = \textsf{\textsf{ReLU}}(\textbf{v} \textbf{W}^\textsf{MLP}_1  + \textbf{b}^\textsf{MLP}_1) \textbf{W}^\textsf{MLP}_2  + \textbf{b}^\textsf{MLP}_2 \in \mathbb{R}^{D}$, where $\textbf{v} \in \mathbb{R}^D$ is the input vector and $\textbf{W}^\textsf{MLP}_1 \in \mathbb{R}^{D \times D_{\textsf{MLP}}}$, $\textbf{b}^\textsf{MLP}_1 \in \mathbb{R}^{D_{\textsf{MLP}}}$, $\textbf{W}^\textsf{MLP}_2 \in \mathbb{R}^{D_{\textsf{MLP}} \times D}$, $\textbf{b}^\textsf{MLP}_2 \in \mathbb{R}^{D}$ the trainable parameters, and~$+$ the addition with broadcasting rules in common NN implementation libraries\repeatfootnote{fn:vector}. Similarly, we have the two affine heads, $\textsf{CLS}(\textbf{v}) = \textbf{v} \textbf{W}^\textsf{CLS} + \textbf{b}^\textsf{CLS} \in \mathbb{R}^C$ and $\textsf{REG}(\textbf{v}) = \textbf{v} \cdot \textbf{W}^\textsf{REG} + b^\textsf{REG} \in \mathbb{R}$, where $\textbf{v} = \widetilde{\textbf{z}^{L}} \in \mathbb{R}^D$ and $\textbf{W}^\textsf{CLS} \in \mathbb{R}^{D \times C}$, $\textbf{b}^\textsf{CLS} \in \mathbb{R}^{C}$, $\textbf{W}^\textsf{REG} \in \mathbb{R}^D$ and $b^\textsf{REG} \in \mathbb{R}$. Finally, we note that a linearized variant of LN is used, i.e., $\textsf{LN}_i(\textbf{v}) = \textbf{w}_i \times (\textbf{v}_i - \mu_{\textbf{v}}) + \textbf{b}_i$, where $\textbf{v} \in \mathbb{R}^D$ is the input vector, $\mu_{\textbf{v}} \in \mathbb{R}$ its mean, $\textsf{LN}_i(\textbf{v})$ the $i$-th element for $i=1,...,D$, and $\textbf{w} \in \mathbb{R}^D$ and $\textbf{b} \in \mathbb{R}^D$ the trainable parameters\repeatfootnote{fn:vector}. Such modification avoids the division over relatively small input variance $\sigma_\textbf{v}$ in the quadratic LN definition, thereby allowing the variables to be bounded more tightly~\cite{shi2020robustness}.

\section{Methodology}
\label{sec:methodology}
Having defined the Sparsemax-based Transformer, we now give the robustness property for verification. Subsequently, we highlight our MIQCP encoding steps and two heuristics that shall accelerate the solving of the encoded MIQCP.

\subsection{Problem Formulation}
\label{subsec:formulation}
We formalize the problem of robustness verification as follows: Let $f(\cdot): \, \mathbb{R}^M \rightarrow \mathbb{R}^N$ denote the NN under verification, $\textbf{x} \in \mathbb{R}^M$ the original data point on which the NN is being verified, and $\textbf{x}' \in \mathbb{R}^M$ a perturbed input which tries to deceive the NN, we write:
\begin{align}
& \min_{\textbf{x}'} \mathcal{D}_p(\textbf{x}', \textbf{x}) \label{eq:obj}  \\
\text{subject to} \; & \textbf{x}' \in \mathcal{B}_p(\textbf{x}) \label{eq:ball}, \\
& \argmax_i(f^\textsf{CLS}_i(\textbf{x})) = \textsf{gt}^{\textsf{CLS}}(\textbf{x}), \; \\
& \argmax_i(f^\textsf{CLS}_i(\textbf{x}')) \neq \textsf{gt}^{\textsf{CLS}}(\textbf{x}) \label{eq:misprediction},
\end{align}
where $\mathcal{D}_p(\cdot,\cdot)$ is the $l_p$-distance with commonly used $p \in \{1, 2, \infty \}$ and  $\mathcal{B}_p(x)= \{ \textbf{x}' \, \big{|} \, \norm{\textbf{x}' - \textbf{x}}_p \leq \epsilon \}$ the $l_p$-norm ball of radius $\epsilon$ around $\textbf{x}$, $\textsf{gt}^{\textsf{CLS}}(x) \in \{1, \dots, C\}$ the ground-truth class label and $f^\textsf{CLS}_i$ the $i$-th element of the classification head output. Conceptually, the optimizer's main task is to find within an admissible perturbation region a perturbed data point closest to the original one and fulfills the misprediction constraints. Due to the space limit, we write only the classification model here, but regression cases can be derived similarly, as suggested by~\cite{bonaert2021fast}.

\subsection{MIQCP Encoding}
\label{subsec:encoding}
In the following, we highlight how to encode Sparsemax (i.e., Algorithm~\ref{alg:sparsemax}, Line~\ref{alg:sparsemax:sort}-~\ref{alg:sparsemax:tau}) into the optimization problem (which will essentially be a MIQCP problem). Encoding methods for other terms in the Transformer (e.g., affine transformation and ReLU) can be seen in~\cite{cheng2017maximum,lomuscio2017approach,tjeng2019evaluating}.

For sorting (Algorithm~\ref{alg:sparsemax}, Line~\ref{alg:sparsemax:sort}), we introduce a binary integer permutation matrix $\textbf{P} \in \{{0,1}\}^{D \times D}$ and encode the following constraints:
\begin{align}
    \sum_{i=1}^D \textbf{P}_{ij} = 1, & \quad \text{for} \; j = 1, \dots, D; \\
    \sum_{j=1}^D \textbf{P}_{ij} = 1, & \quad \text{for} \; i = 1, \dots, D; \\
    \hat{\textbf{u}} = \textbf{P}\textbf{u}, & \\
    \hat{\textbf{u}}_i \geq \hat{\textbf{u}}_{i+1}, & \quad \text{for} \; i = 1, \dots, D-1,
\end{align}
where $\textbf{u} \in \mathbb{R}^D$ is the (perturbed) input vector at Sparsemax and $\hat{\textbf{u}}$ the sorted output. 

For calculating the support (Algorithm~\ref{alg:sparsemax}, Line~\ref{alg:sparsemax:find}), we first define a vector $\rho \in \mathbb{R}^D$, where $\rho_k = 1 + k\hat{\textbf{u}}_k - \sum_{j=1}^{k} \hat{\textbf{u}}_j$ ($k = 1, \dots, D$), and then introduce another binary integer vector $\zeta \in \{0, 1 \}^D$ such that:
\begin{align}
    \zeta_k = 
    \begin{cases}
    1, \quad \text{if} \; \rho_k > 0; \\
    0, \quad \text{otherwise}.
    \end{cases}
    \label{eq:zeta}
\end{align}
It can be seen that finding the support is equivalent to summing up the vector as $s(\hat{\textbf{u}}) = \sum_{k=1}^D \zeta_k$. However, we actually need to implement the step function in~\eqref{eq:zeta} in the MIQCP. For this, we introduce the Big-$M$ method~\cite{grossmann2002review} with large positive constants $M_k^{+/-}$ and a small positive constant $\eta$ (e.g., $10^{-6}$) such that:
\begin{align}
    \rho_k & \leq M_k^{+} \times \zeta_k, \label{eq:M_pos} \\
    -\rho_k + \eta & \leq M_k^{-} \times (1 - \zeta_k). \label{eq:M_neg}
\end{align}
We now provide two lemmas to explain how $M_k^{+/-}$ are set. 

\begin{lemma}
For all $k=1, \dots, D$, the smallest value for the Big-$M$ encoding in \eqref{eq:M_pos} is $\prescript{opt}{}M_k^{+} = 1$.
\label{lem:big_m_pos}
\end{lemma}
\begin{proof}
    We first rewrite $\rho_k = 1 + k \hat{\textbf{u}}_k - \sum_{j=1}^{k} \hat{\textbf{u}}_j = 1 + \sum_{j=1}^{k} (\hat{\textbf{u}}_k - \hat{\textbf{u}}_j)$. Now, with the sorting result (i.e., $\hat{\textbf{u}}_1 \geq \hat{\textbf{u}}_2 \geq \ldots \geq \hat{\textbf{u}}_D$), it follows that $1 = \rho_1 \ge \rho_2 \ge \dots \ge \rho_D$, hence the lemma.
    \hfill$\blacksquare$
\end{proof}

\begin{lemma}
For all $k=1, \dots, D$, let the input of Sparsemax be bounded as $\textbf{u}_k \in [\underline{\textbf{u}_k}, \overline{\textbf{u}_k}]$ and $\eta=10^{-6}$. We first define $\lambda \in \mathbb{R}^D: \lambda_k =  1 + (k-1)(\underline{\textbf{u}} - \overline{\textbf{u}})$, where $\underline{\textbf{u}} = \min (\underline{\textbf{u}_1}, \dots, \underline{\textbf{u}_D})$ and $\overline{\textbf{u}} = \max (\overline{\textbf{u}_1}, \dots, \overline{\textbf{u}_D})$. Then, the smallest value for the Big-$M$ encoding in \eqref{eq:M_neg} is $\prescript{opt}{}M_k^{-} = |\lambda_k|+\eta$, if $\lambda_k \le 0$; otherwise, \eqref{eq:M_neg} needs not be implemented.
\label{lem:big_m_neg}
\end{lemma}
\begin{proof}
    After sorting on $\textbf{u}$, we can only rely on vector-wise bounds for estimating $\hat{\textbf{u}}$, i.e., $\hat{\textbf{u}}_k \in  [ \underline{\textbf{u}}, \overline{\textbf{u}} ]$. Considering the result of sorting (i.e., $\hat{\textbf{u}}_1 \geq \hat{\textbf{u}}_2 \geq \ldots \geq \hat{\textbf{u}}_D$), we can then derive $\rho_k = 1 + k \hat{\textbf{u}}_k - \sum_{j=1}^{k} \hat{\textbf{u}}_j = 1 + \sum_{j=1}^{k} (\hat{\textbf{u}}_k - \hat{\textbf{u}}_j) \ge 1 + (k-1)(\underline{\textbf{u}} - \overline{\textbf{u}}) = \lambda_k$. Now, there are two cases: If $\lambda_k \le 0$, then the smallest value for \eqref{eq:M_neg} is $\prescript{opt}{}M_k^-=|\lambda_k|+\eta$. Otherwise, $\rho_k \ge \lambda_k > 0$ and \eqref{eq:M_neg} is not needed.
    \hfill$\blacksquare$
\end{proof}
Lastly, Algorithm~\ref{alg:sparsemax} Line~\ref{alg:sparsemax:tau} can be encoded with a linear constraint for the summation term and a quadratic constraint for the division. As such, we arrive at a plain MIQCP encoding for quantifying the Sparsemax-based Transformer's robustness.

\subsection{Acceleration Heuristics}
\label{subsec:heuristics}
As related works indicate, one usually needs several acceleration heuristics to solve an encoded MIP problem efficiently~\cite{cheng2017maximum,lomuscio2017approach,tjeng2019evaluating}. We present our proposals in this section.

\subsubsection{Interval Analysis}
Interval analysis has been widely studied and proven effective in aiding MIP solving~\cite{cheng2017maximum,tjeng2019evaluating}. The central idea is that with tight interval bounds propagated across the network, non-linear functions such as ReLU or max can be constrained to certain behaviors, thus avoiding undetermined variables. We extend this line of thought by finding a novel bounding technique over Sparsemax\footnote{This technique also applies to Softmax.}. Denoting the activation function as $\sigma$, (perturbed) input vector as $\textbf{z} \in \mathbb{R}^{D}$, output vector as $\textbf{a}$, upper and lower bounds as upper and lower bars, we have:
\begin{align}
	\begin{split}
		\textbf{a}_i \in [\underline{\textbf{a}_i}, \; \overline{\textbf{a}_i} ] = & [ \sigma_i (\overline{\textbf{z}_1},  \dots, \underline{\textbf{z}_i}, \dots, \overline{\textbf{z}_D}), \\
		& \; \sigma_i (\underline{\textbf{z}_1},  \dots, \overline{\textbf{z}_i}, \dots, \underline{\textbf{z}_D}) ],
	\end{split}
\end{align}
where $i=1,\dots,D$. Essentially, this formula stems from the observation that the lower (upper) bound of the output of an element can be calculated by applying the activation to a vector consisting of this element's input lower (upper) bound and other elements' input upper (lower) bounds. In our evaluation, we generally see much tighter intervals (usually one-tenth of the simple probability interval $[0, 1]$) with this technique. 

\subsubsection{Progressive Verification with Norm-Space Region Partitioning}
Region partitioning shares a similar goal to interval analysis, attempting to tighten variable bounds and generate faster solutions. Related works have focused on how to do partitioning and prioritizing~\cite{everett2021robustness}. Our work differs slightly by exploiting the formulation of the optimization problem and the observation that adversarial examples usually appear close to the clean inputs. Hence, we propose a progressive verification procedure based on norm-space region partitioning. To illustrate, given $\epsilon$, the radius of the $l_p$-norm ball, and a partition step $0 < \epsilon_{step} \leq \epsilon$, we first create a sub-region with a lower bound $\epsilon_{min} = 0$ and an upper bound $\epsilon_{max} = \epsilon_{min} + \epsilon_{step}$ and then run verification on this sub-region. If the verifier cannot find a solution in the current sub-region, we move on to the next one by setting $\epsilon_{min} \mathrel{+}= \epsilon_{step}$ and $\epsilon_{max} \mathrel{+}= \epsilon_{step}$ until the entire admissible region is covered. As such, we generally obtain a tighter interval for the perturbed input variable (taking $p=1$, for instance):
\begin{align}
& 0 \leq \epsilon_{min} \leq \norm{\textbf{x}' - \textbf{x}}_1 = \norm{\delta}_1 \leq \epsilon_{max} \leq \epsilon \label{eq:region_partitioning} \\ 
\Rightarrow & \quad \epsilon_{min} \leq \sum_{d=1}^{D} | \delta_d | \leq \epsilon_{max} \\
\Rightarrow & \quad 0 \leq | \delta_d | \leq \epsilon_{max} \\
\Rightarrow & \quad -\epsilon_{max} \leq \delta_d  \leq \epsilon_{max} \\
\Rightarrow & \quad \textbf{x}'_d \in \Big[ \underline{\textbf{x}'_d},\ \;  \overline{\textbf{x}'_d} \Big] = \big[ \textbf{x}_d-\epsilon_{max}, \; \textbf{x}_d+\epsilon_{max} \big],
\end{align}
where $d=1,\dots,D$. The tightness of the variable intervals depends on the value of $\epsilon_{max}$. If $\epsilon_{max}$ grows towards $\epsilon$, the variable intervals shall fall back to the original regions. Nonetheless, considering adversarial examples often appear closely around the clean inputs, we conjecture that early sub-regions would already contain one, offering a high possibility for a quick solution. 

\begin{algorithm}[b!]
\caption{Verify an NN model with norm-space region partitioning}
\label{alg:verify}
\begin{algorithmic}[1]
\Procedure{VerifyModel}{$f, \textbf{x}, \textsf{gt}(\textbf{x}), \epsilon, \epsilon_{step}, \epsilon_{min}, \epsilon_{max}, t_{limit}$}
    \State Encode $f, \textbf{x}, \textsf{gt}(\textbf{x})$ into an MIQCP $\mathcal{M}$ with $\epsilon_{min}$ and $\epsilon_{max}$
    \State Solve $\mathcal{M}$ with Gurobi under $t_{limit}$ and obtain an interim tuple $(s, n, obj, \textbf{x}', gap, t_{exec})$ \label{alg:verify:interim}
    \If{$s$ is \textsf{optimal}}
        \State \textbf{return} (\textsf{OPT}, $obj$, $\textbf{x}'$)
    \ElsIf{$s$ is \textsf{timeout} \textbf{and} $n > 0$}
        \State \textbf{return} (\textsf{SAT}, $obj$, $\textbf{x}'$, $\alpha$)
    \ElsIf{$s$ is \textsf{timeout} \textbf{and} $n = 0$} \label{alg:verify:undetermined_1}
        \State \textbf{return} (\textsf{UNDTM}, $\epsilon_{min}$) \label{alg:verify:undetermined_2}
    \ElsIf{$s$ is \textsf{infeasible}}
        \State $t_{limit}$ -= $t_{exec}$ 
        \State $\epsilon_{min}$ += $\epsilon_{step}$
        \State $\epsilon_{max}$ = $\min$($\epsilon_{max}$+$\epsilon_{step}$, $\epsilon$)
        \If{$\epsilon_{min} > \epsilon$}
            \State \textbf{return} (\textsf{UNSAT})
        \Else
            \State \textbf{return} \textsc{VerifyModel}($f, \textbf{x}, \textsf{gt}(\textbf{x}), \epsilon, \epsilon_{step}, \epsilon_{min}, \epsilon_{max}, t_{limit})$
        \EndIf
    \EndIf
\EndProcedure
\Input{$f$, $\textbf{x}$, $\textsf{gt}(\textbf{x})$, $\epsilon$, $\epsilon_{step}$, $t_{limit}$}
\State \textbf{Initialize: } $\epsilon_{min} \gets 0$, $\epsilon_{max} \gets \epsilon_{step}$
\State $verification\_result$ = $\textsc{VerifyModel}(f, \textbf{x}, \textsf{gt}(\textbf{x}), \epsilon, \epsilon_{step}, \epsilon_{min}, \epsilon_{max}, t_{limit})$
\Output{$verification\_result$}
\end{algorithmic}
\end{algorithm}

We summarize the progressive procedure in Algorithm~\ref{alg:verify}. Additionally, to prevent the verifier from executing unboundedly, we place a pre-defined time limit $t_{limit}$ on the algorithm. We employ Gurobi 9.5~\cite{gurobi2021gurobi} to solve the progressively encoded MIQCPs $\mathcal{M}$. As shown in Line~\ref{alg:verify:interim}, for each call on a sub-region, the solver returns an interim tuple including the optimization status $s$ (being \textsf{optimal}, \textsf{timeout} or \textsf{infeasible}), solution count $n$, objective $obj$, counterexample $\textbf{x}'$, solving gap $\alpha$ and execution time $t_{exec}$. Based on the interim tuple, we check if the verification is done optimally (\textsf{OPT}), timed out with a sub-optimal solution (\textsf{SAT}), timed out with no solution (\textsf{UNDTM}), unsatisfied (\textsf{UNSAT}), or to be continued with the next sub-region. Each of these conditions will be appended with corresponding data from the interim tuple to form the final verification result. Notably, apart from fastening the verification process, another advantage of such a progressive procedure is that one can still attain a good lower bound of the optimal objective if a timeout occurs (i.e., Algorithm~\ref{alg:verify}, Line~\ref{alg:verify:undetermined_2}).
\section{Experimental Results and Discussions}
\label{sec:experiments}
This section presents the experimental results of the proposed method and techniques. We first introduce the Lane Departure Warning (LDW) application, then show an accuracy benchmark, and finally present an ablation study on the heuristics and a robustness comparison.
\vspace{-1mm}

\subsection{Lane Departure Warning}
\label{subsec:lane_departure_warning}
The LDW application performs a time-series joint classification and regression task~\cite{mahajan2020prediction}. For training and evaluating the NNs, we utilize the High-D dataset\footnote{The utilization of the High-D dataset in this paper is for knowledge dissemination and scientific publication and is not for commercial use.}~\cite{krajewski2018highd}, a drone-recorded bird's-eye-view highway driving dataset. For each vehicle in the recordings, we process raw data into trajectory information, including the past ten steps of time-wise features $\textbf{x} \in \mathbb{R}^{10\times14}$ (spanning across one second), the direction $\textsf{gt}^\textsf{CLS} \in \{0, 1, 2\}$ of the lane departure (where $0$ is no departure, $1$ a left departure, and $2$ a right departure), and the time $\textsf{gt}^\textsf{REG} \in [0, 1]$ to the lane departure. More specifically, the 14 features include left and right lane existence (2), ego distance to lane center (1), longitudinal and latitudinal velocities and accelerations (4), and time-to-collision to surrounding vehicles excluding the following one (7). Based on these features, the NN's task is to predict a potential lane departure direction (i.e., classification) and timing (i.e., regression) up to one second in the future.
\vspace{-1mm}

\subsection{Accuracy Benchmark}
\label{subsec:accuracy_benchmark}
As a side experiment, we evaluate the accuracy of different NNs for LDW. For classification, the model is accurate if the predicted direction matches the ground-truth direction; for regression, it is accurate if the predicted timing falls within 0.1 second from the ground-truth timing. For the Transformers, we set $H=2, D_H=4, D_\textsf{MLP}=8$ (as per Section~\ref{sec:preliminaries}). For the MLPs, we replace the MSA within the Transformer with another MLP of hidden-layer dimension $D_{\textsf{MLP}}=16$, resulting in similar numbers of network parameters. We implement the NNs with PyTorch~\cite{paszke2019pytorch}, train them using the Adam optimizer and a fixed learning rate of 0.003 for 50 epochs, and report the best results in Table~\ref{tab:training_results}.

As observed, even for a relatively small application (considering the variable dimensions), the Transformers generally perform better than MLPs regarding accuracy. Additionally, NNs with piece-wise linear activation functions (i.e., Sparsemax and ReLU) are on a par with, if not stronger than, the ones with Softmax or Tanh. This result corresponds well with the original Sparsemax paper~\cite{martins2016softmax} and justifies the piece-wise linear activation functions in exchange for better verifiability (as also suggested in~\cite{shi2020robustness}). 

\setBoldness{0.6}
\begin{table}[t!]
\caption{Accuracy of various NNs in LDW: Transformers tend to be more accurate than MLPs. $L$ denotes the number of layers in the NN as defined in Section~\ref{sec:preliminaries}. For LN, 2 denotes the quadratic variant, 1 the linear variant, and 0 no layer normalization. Within each column, we mark the best model overall in bold and the best model by its type (i.e., Transformer or MLP) in italic fonts.}
\renewcommand{\arraystretch}{0.95}
\setlength\tabcolsep{3pt}
\centering
\footnotesize
\resizebox{0.6\columnwidth}{!}{
\begin{tabular}{|c|c|c|cc|cc|}
    \hline
    \multirow{2}{*}{Network} &
    \multirow{2}{*}{Activation} &
    \multirow{2}{*}{LN} &
    \multicolumn{2}{c|}{$L=1$} & 
    \multicolumn{2}{c|}{$L=2$} \\
    & & & $\textsf{CLS}$ & $\textsf{REG}$ & $\textsf{CLS}$ & $\textsf{REG}$ \\
    \hline
    
    \multirow{6}{*}{Transformer} & 
    \multirow{3}{*}{Sparsemax} & 2 & 98.01\% & 87.24\% & 98.50\% & 91.49\% \\
    & & 1 & 98.31\% & 88.46\% & 98.52\% & 91.39\% \\
    & & 0 & \textit{{\fbseries 98.34\%}} & \textit{{\fbseries 90.38\%}} & 98.54\% & \textit{{\fbseries 92.31\%}} \\
    \cline{2-7}
    & \multirow{3}{*}{Softmax} & 2 & 98.01\% & 87.11\% & 98.55\% & 91.88\% \\
    & & 1 & 97.95\% & 87.56\% & 98.61\% & 91.52\% \\
    & & 0 & 98.21\% & 88.03\% & \textit{{\fbseries 98.64\%}} & 91.21\% \\
    \cline{1-7} 
    
    \multirow{6}{*}{MLP} & 
    \multirow{3}{*}{ReLU} & 2 & 97.62\% & 83.81\% & 97.81\% & 84.70\% \\
    & & 1 & 97.46\% & 83.55\% & 97.89\% & 85.04\% \\
    & & 0 & 97.68\% & \textit{84.27\%} & 97.92\% & 84.81\% \\
    \cline{2-7}
    & \multirow{3}{*}{Tanh} & 2 & \textit{97.80\%} & 83.20\% & \textit{97.99\%} & 84.63\% \\
    & & 1 & 97.32\% & 82.93\% & 97.69\% & 84.81\% \\
    & & 0 & 97.23\% & 83.02\% & 97.85\% & \textit{85.14\%} \\
    \hline
\end{tabular}
}
\vspace{-1mm}
\label{tab:training_results}
\end{table}

\begin{table*}[h]
\centering
\renewcommand{\arraystretch}{0.95}
\setlength\tabcolsep{3pt}
\caption{Ablation study on the proposed acceleration heuristics, including interval analysis without and with Sparsemax bounding (IA and IA-$\sigma$) and norm-space region partitioning with different epsilon steps (RP-\textit{0.001}, RP-\textit{0.005} and RP-\textit{0.01}): Our two heuristics can give a total speedup of up to one order of magnitude. We set $\epsilon=0.05$ in all experiments. We mark the best-performing numbers of each sub-group in italic fonts if they are better than the control and further mark the best-performing numbers across all heuristics in bold fonts. For the three listed samples, we specify optimally solved cases as \textsf{OPT} and augment undetermined cases (\textsf{UNDTM}) with a lower bound of the minimum adversarial distortion (better if higher) and satisfied cases (\textsf{SAT}) with the MIQCP solution gap (better if lower).}
\label{tab:ablation}
\resizebox{\columnwidth}{!}{
    \begin{tabular}{rrrrrrrrrr}
    \toprule
    \multicolumn{1}{r}{\multirow{2}{*}{\begin{tabular}[c]{@{}c@{}}Techniques\end{tabular}}}
    & \multicolumn{3}{c}{Time elpased (s)} 
    & \multicolumn{3}{c}{Nodes explored}
    & \multicolumn{3}{c}{Random samples} \\
    \cmidrule(lr){2-4} \cmidrule(lr){5-7} \cmidrule(lr){8-10}
    & \multicolumn{1}{c}{Mean} & \multicolumn{1}{c}{Best} & \multicolumn{1}{c}{Worst} & \multicolumn{1}{c}{Mean} & \multicolumn{1}{c}{Best} & \multicolumn{1}{c}{Worst} & \multicolumn{1}{c}{No. 1} & \multicolumn{1}{c}{No. 2} & \multicolumn{1}{c}{No. 3} \\
    \midrule
    \vspace{2mm}
    \textit{Control}
    & 2367.14 & 89.63 & 3602.50 & 9326275 & 118329 & 16395255 & (\textsf{UNDTM}, 0.0) & (\textsf{SAT}, 1.0) & \textsf{OPT} \\ 
    IA
    & 1703.65 & 330.26 & 3605.29 & 5370138 & 911639 & 13103193 & (\textsf{UNDTM}, 0.0) & \textsf{OPT} & \textsf{OPT} \\
    \vspace{2mm}
    IA-$\sigma$
    & \textit{950.37} & 127.33 & \textit{2936.72} & \textit{3061462} & 247916 & \textit{10298216} & \textsf{OPT} & \textsf{OPT} & \textsf{OPT}  \\
    RP-\textit{0.001}
    & 755.66 & \textit{7.08} & 3609.83 & 849968 & \textit{11517} & 4058680 & (\textsf{UNDTM}, 0.005) & \textsf{OPT} & \textsf{OPT} \\ 
    RP-\textit{0.005} 
    & \textit{278.20} & 15.41 & \textit{1276.86} & \textit{{\fbseries 276310}} & 13734 & \textit{{\fbseries 1255818}} & \textsf{OPT} & \textsf{OPT} & \textsf{OPT} \\ 
    \vspace{2mm}
    RP-\textit{0.01} 
    & 793.76 & 51.39 & 1556.72 & 1825979 & 70243 & 3429947 & \textsf{OPT} & \textsf{OPT} & \textsf{OPT} \\
    IA-$\sigma$+RP-\textit{0.001} 
    & 852.67 & \textit{{\fbseries 5.36}} & 3608.60 & 1689582 & 14557 & 6573066 & (\textsf{UNDTM}, 0.004) & \textsf{OPT} & \textsf{OPT} \\ 
    IA-$\sigma$+RP-\textit{0.005} 
    & 296.90 & 7.74 & 1413.32 & 505555 & \textit{{\fbseries 7156}} & 2413044 & \textsf{OPT} & \textsf{OPT} & \textsf{OPT} \\ 
    IA-$\sigma$+RP-\textit{0.01} 
    & \textit{{\fbseries 222.68}} & 59.56 & \textit{{\fbseries 644.10}} & \textit{494973} & 99642 & \textit{1479260} & \textsf{OPT} & \textsf{OPT} & \textsf{OPT} \\ 
    \bottomrule
    \end{tabular}
}
\end{table*}

\subsection{Ablation Study}
\label{subsec:ablation}
We now conduct an ablation study on the acceleration heuristics described in Section~\ref{subsec:heuristics}, using the Sparsemax-based Transformer with $L=1$ and linearized LN. During our verification, we only allow the final token of the input variable $\textbf{x}$ to be perturbed, resulting in an encoded MIQCP with roughly 4000 linear constraints, 700 quadratic constraints, 400 general constraints (e.g., $\max$ or absolute operations) and 2500 binary variables. We test the heuristics with five random samples from the curated dataset and summarize the results in Table~\ref{tab:ablation}. All experiments are run with an Intel i9-10980XE CPU @ 3.0GHz using 18 threads and 20 GB of RAM.

It is observed that the novel activation bounding technique is effective. Likewise, progressive verification with norm-space region partitioning further reduces verification time. However, the best epsilon step might vary among test cases as its magnitude does not necessarily correlate to the best-case verification speed. A reason behind this might be the excessive number of MIP instances created by smaller epsilon steps. Notably, combining interval analysis and progressive verification delivers a speedup of approximately an order of magnitude.
\vspace{-2mm}

\subsection{NN Robustness Comparisons}
\label{subsec:robustness_comparison}
For robustness comparisons, we first verify and compare the robustness of the aforementioned Sparsemax-based Transformer and the similar-sized ReLU-based MLP. They respectively achieve 98.31\% and 97.46\% in accuracy for classification in Section~\ref{subsec:accuracy_benchmark}\footnote{Verifying the MLP follows similar steps in Section~\ref{sec:methodology} except that the encoding is relatively simple and can be solved by Mixed Integer Linear Programming (MILP).}. We set $\epsilon=0.03$ and $p=1$ for the admissible $\ell_p$-ball\footnote{The admissible perturbation region can be derived from input feature values analytically for better physical interpretability. For example, we can set $\epsilon$ as the normalized value of ego car lateral acceleration, considering it a decisive feature for LDW. Perturbations in this context can stem from sensor noises or hardware faults. The binary features are not perturbed.} and enable IA-$\sigma$+RP-\textit{0.01} from the previous section during verification. We collect results from 60 random data points (on which the Transformer and MLP predict correctly before perturbation) and plot the robustness comparison diagram in Fig.~\ref{fig:performance}. Overall, excluding 2 points where they draw, the Transformer is more robust than the MLP on 26 points yet less robust on 32. The result shows that the Transformer is not necessarily more robust despite the higher accuracy.

\begin{figure}[h]
\centering
\resizebox{0.5\columnwidth}{!}{
    \begin{tikzpicture}
        \begin{axis}[
            xlabel=Transformer,
            ylabel=MLP,
            xmin=0,
            ymin=0,
            xmax=0.035,
            ymax=0.035,
            xtick={0, 0.01, 0.02,  0.03},
            xticklabels={$0$, $0.01$, $0.02$, $\geq 0.03$},
            ytick={0.01, 0.02, 0.03},
            yticklabels={$0.01$, $0.02$, $\geq 0.03$},
            scaled ticks=false, 
            tick label style={/pgf/number format/fixed},
            legend style={at={(1.1,0)}, anchor=south west}
            ]
            \addplot[only marks, mark=o, mark size=5pt] 
            table[x=ANN,y=MLP,col sep=comma] {files/performance_exact.csv};
            \addplot[only marks, mark=x, mark size=5pt] 
            table[x=ANN,y=MLP,col sep=comma] {files/performance_lower.csv};
            \addplot+[no markers, dashed]  
            coordinates {(0,0) (0.035, 0.035)};
            \legend{exact, lower}
        \end{axis}
    \end{tikzpicture}
}
\caption{Robustness of the Transformer and MLP on 60 random data points (better if larger). The dashed line highlights where they perform equally. There are 2 points on the dashed line, 26 in the lower triangle and 32 in the upper one. Since verifying the Transformer still takes much time, we report the lower bounds of the robustness values for the data points requiring more than one hour to verify. This means the points marked by ``lower" can be further pushed to the right if the verifier is given more time. Nonetheless, this does not affect the overall observation as we ensure all such cases are points where the Transformer performs more robustly already.}
\label{fig:performance}
\vspace{-2mm}
\end{figure}
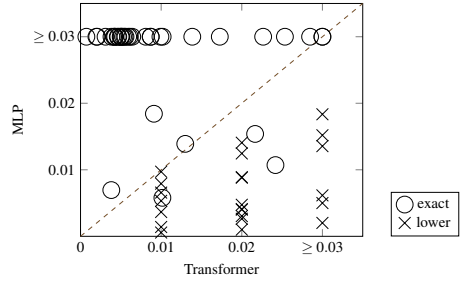

We further train and verify the robustness of five Transformers and five MLPs, each on 20 data points, giving 100 samples for each NN type. For faster verification, we set $\epsilon=0.01$ as the threshold for being robust and surprisingly find that the MLPs are verified robust on all data points. In contrast, the Transformers are verified robust on 69 data points and have a mean robustness value of $0.0041$ (excluding the robust ones). Our finding is a counterexample to several testing results on vision tasks, which suggest that Transformers are generally more robust~\cite{bhojanapalli2021understanding,shao2021adversarial}. Accordingly, it is believed that NNs may perform differently in diverse domain tasks, and it is vital to rigorously evaluate both accuracy and robustness before deploying NN-based applications.
\vspace{-1mm}

\section{Conclusion}
\label{sec:conclusion}
\vspace{-1mm}
This paper works towards exact robustness verification for ATNs. Specifically, we focus on Sparsemax-based ATNs, encode them into a MIQCP problem, and propose accelerating heuristics for solving the problem faster. When applied, our proposals fasten the verification process roughly one order of magnitude. We conduct experiments with a Lane Departure Warning application and find that ATNs are less robust than MLPs in our settings.

This initial study opens some interesting directions for further exploration. First, as we consider only Sparsemax-based ATNs, we are exploring further techniques to improve the bounds on Softmax that may facilitate its exact robustness verification. Second, we evaluate only small-scale networks (approximately 1600 neurons) on one dataset (with inputs of $14 \times 10$ dimensions). Whether our observations remain true for larger-scale networks and more datasets is yet to be explored. Third, our problem formulation only examines point-wise robustness verification for NNs. Such analyses can be combined with systematic sampling and testing methods to give formal and statistical guarantees on safety-critical applications. Lastly, our verification requires ground truths and works only in design time. How to utilize the studied techniques in a run-time setting remains an open question.
\vspace{-1mm}

\bibliographystyle{splncs04}
\bibliography{ref}

\begin{thebibliography}{10}
\providecommand{\url}[1]{\texttt{#1}}
\providecommand{\urlprefix}{URL }
\providecommand{\doi}[1]{https://doi.org/#1}

\bibitem{bhojanapalli2021understanding}
Bhojanapalli, S., Chakrabarti, A., Glasner, D., Li, D., Unterthiner, T., Veit,
  A.: Understanding robustness of transformers for image classification. In:
  ICCV (2021)

\bibitem{bojarski2016end}
Bojarski, M., Del~Testa, D., Dworakowski, D., Firner, B., Flepp, B., Goyal, P.,
  Jackel, L.D., Monfort, M., Muller, U., Zhang, J., Zhang, X., Zhao, J., Zieba,
  K.: End to end learning for self-driving cars (2016)

\bibitem{bonaert2021fast}
Bonaert, G., Dimitrov, D.I., Baader, M., Vechev, M.: Fast and precise
  certification of transformers. In: PLDI (2021)

\bibitem{cheng2017maximum}
Cheng, C.H., Nührenberg, G., Ruess, H.: Maximum resilience of artificial
  neural networks. In: ATVA (2017)

\bibitem{cruise2021underthehood}
Cruise: {Cruise Under the Hood 2021}, \url{https://youtu.be/uJWN0K26NxQ?t=1342}

\bibitem{dosovitskiy2021image}
Dosovitskiy, A., Beyer, L., Kolesnikov, A., Weissenborn, D., Zhai, X.,
  Unterthiner, T., Dehghani, M., Minderer, M., Heigold, G., Gelly, S.,
  Uszkoreit, J., Houlsby, N.: An image is worth 16x16 words: Transformers for
  image recognition at scale. In: ICLR (2021)

\bibitem{ehlers2017formal}
Ehlers, R.: Formal verification of piece-wise linear feed-forward neural
  networks. In: ATVA (2017)

\bibitem{ec2021aiact}
{European Commission}: {EU AI Act} (2021),
  \url{https://artificialintelligenceact.eu/}

\bibitem{everett2021robustness}
Everett, M., Habibi, G., How, J.P.: Robustness analysis of neural networks via
  efficient partitioning with applications in control systems. IEEE Control
  Syst. Lett.  \textbf{5},  2114--2119 (2021)

\bibitem{gehr2018ai2}
Gehr, T., Mirman, M., Drachsler-Cohen, D., Tsankov, P., Chaudhuri, S., Vechev,
  M.: Ai2: Safety and robustness certification of neural networks with abstract
  interpretation. In: SP (2018)

\bibitem{goodfellow2015explaining}
Goodfellow, I.J., Shlens, J., Szegedy, C.: Explaining and harnessing
  adversarial examples. In: ICLR (2015)

\bibitem{grossmann2002review}
Grossmann, I.E.: Review of nonlinear mixed-integer and disjunctive programming
  techniques. Optim. Eng.  \textbf{3},  227--252 (2002)

\bibitem{gurobi2021gurobi}
{Gurobi Optimization, LLC}: Gurobi optimizer reference manual (2021)

\bibitem{hu2022human}
Hu, B.C., Marsso, L., Czarnecki, K., Salay, R., Shen, H., Chechik, M.: If a
  human can see it, so should your system: {R}eliability requirements for
  machine vision components. In: ICSE (2022)

\bibitem{huang2020survey}
Huang, X., Kroening, D., Ruan, W., Sharp, J., Sun, Y., Thamo, E., Wu, M., Yi,
  X.: A survey of safety and trustworthiness of deep neural networks:
  Verification, testing, adversarial attack and defence, and interpretability.
  Comput. Sci. Rev.  \textbf{37},  100270 (2020)

\bibitem{huang2017safety}
Huang, X., Kwiatkowska, M., Wang, S., Wu, M.: Safety verification of deep
  neural networks. In: CAV (2017)

\bibitem{katz2017reluplex}
Katz, G., Barrett, C., Dill, D., Julian, K., Kochenderfer, M.: Reluplex: An
  efficient smt solver for verifying deep neural networks. In: CAV (2017)

\bibitem{krajewski2018highd}
Krajewski, R., Bock, J., Kloeker, L., Eckstein, L.: The highd dataset: A drone
  dataset of naturalistic vehicle trajectories on german highways for
  validation of highly automated driving systems. In: ITSC (2018)

\bibitem{lomuscio2017approach}
Lomuscio, A., Maganti, L.: An approach to reachability analysis for
  feed-forward relu neural networks (2017)

\bibitem{mahajan2020prediction}
Mahajan, V., Katrakazas, C., Antoniou, C.: Prediction of lane-changing
  maneuvers with automatic labeling and deep learning. TRR Journal
  \textbf{2674},  336--347 (2020)

\bibitem{martins2016softmax}
Martins, A.F.T., Astudillo, R.F.: From softmax to sparsemax: A sparse model of
  attention and multi-label classification. In: ICML (2016)

\bibitem{paszke2019pytorch}
Paszke, A., Gross, S., Massa, F., Lerer, A., Bradbury, J., Chanan, G., Killeen,
  T., Lin, Z., Gimelshein, N., Antiga, L., Desmaison, A., Kopf, A., Yang, E.,
  DeVito, Z., Raison, M., Tejani, A., Chilamkurthy, S., Steiner, B., Fang, L.,
  Bai, J., Chintala, S.: Pytorch: An imperative style, high-performance deep
  learning library. In: NeurIPS (2019)

\bibitem{poretschkin2023ai}
Poretschkin, M., Schmitz, A., Akila, M., Adilova, L., Becker, D., Cremers,
  A.B., Hecker, D., Houben, S., Mock, M., Rosenzweig, J., Sicking, J., Schulz,
  E., Voss, A., Wrobel, S.: {AI} assessment catalog (2023),
  \url{https://www.iais.fraunhofer.de/en/research/artificial-intelligence/ai-assessment-catalog.html}

\bibitem{shao2021adversarial}
Shao, R., Shi, Z., Yi, J., Chen, P.Y., Hsieh, C.J.: On the adversarial
  robustness of vision transformers. In: UCCV (2021)

\bibitem{shi2020robustness}
Shi, Z., Zhang, H., Chang, K.W., Huang, M., Hsieh, C.J.: Robustness
  verification for transformers. In: ICLR (2020)

\bibitem{su2019one}
Su, J., Vargas, D.V., Sakurai, K.: One pixel attack for fooling deep neural
  networks. IEEE Trans. Evol. Comput.  \textbf{23},  828–841 (2019)

\bibitem{tesla2022aiday}
Tesla: {Tesla AI Day 2022},
  \url{https://www.youtube.com/live/ODSJsviD_SU?feature=share&t=4464}

\bibitem{tjeng2019evaluating}
Tjeng, V., Xiao, K., Tedrake, R.: Evaluating robustness of neural networks with
  mixed integer programming. In: ICLR (2019)

\bibitem{vaswani2017attention}
Vaswani, A., Shazeer, N., Parmar, N., Uszkoreit, J., Jones, L., Gomez, A.N.,
  Kaiser, L., Polosukhin, I.: Attention is all you need. In: NeurIPS (2017)

\bibitem{wang2021betacrown}
Wang, S., Zhang, H., Xu, K., Lin, X., Jana, S., Hsieh, C.J., Kolter, J.Z.:
  Beta-crown: Efficient bound propagation with per-neuron split constraints for
  complete and incomplete neural network verification (2021)

\bibitem{wong2018provable}
Wong, E., Kolter, J.Z.: Provable defenses against adversarial examples via the
  convex outer adversarial polytope. In: ICML (2018)

\bibitem{xiong2020layer}
Xiong, R., Yang, Y., He, D., Zheng, K., Zheng, S., Xing, C., Zhang, H., Lan,
  Y., Wang, L., Liu, T.Y.: On layer normalization in the transformer
  architecture. In: ICLR (2020)

\end{thebibliography}

\end{document}